%% file: X2Streaming_TTS_arxiv.tex
\documentclass[11pt,letterpaper]{article}
\usepackage{arxiv}
\usepackage[hyphens]{url}
\usepackage{graphicx}
\usepackage{natbib}
\usepackage{caption}
\usepackage{booktabs}
\usepackage{multirow}
\usepackage{pifont}
\usepackage{amsmath}
\usepackage{amssymb}
\usepackage{amsthm}
\usepackage{placeins}
\newtheorem{theorem}{Theorem}

\theoremstyle{remark}

\renewcommand{\shorttitle}{X2Streaming-TTS}
\date{}

\title{X2Streaming-TTS: Causal Token-Level Text-to-Speech from Streaming Text with Speech-State Inheritance}
\author{
Rime Wen$^{*}$,
Zehan Liu$^{*}$,
Shawn Qin,
Lights Shi,
Roy Gan,
Hao Wang$^{\dagger}$,
Qian Wang \\[0.6em]
X Square Robot \\[0.8em]
{\normalsize
$^{*}$These authors contributed equally.\;
$^{\dagger}$Corresponding author.}
}

\begin{document}
\maketitle

\begin{abstract}
Streaming text-to-speech is essential for low-latency spoken dialogue systems, yet many systems wait for sentence-level text and are therefore only pseudo-streaming. True token-level synthesis must generate speech from uncertain prefixes while maintaining perceptual continuity over an unbounded stream with bounded context. We present X2Streaming-TTS, a causal TTS framework that consumes asynchronously arriving text tokens and emits speech without accessing future input. To handle uncertain prefixes, we introduce causal commitment, which keeps ambiguous expressions provisional through uncertainty-aware buffering and performs capacity-adaptive, punctuation-aware segmentation. To preserve acoustic continuity, we further introduce causal speech-state inheritance, which carries the complete Code2Wav state and selected historical Talker states across segment boundaries. Together with an attention prior constraint, it blocks access to future positions while retaining bounded acoustic context. Experiments show that X2Streaming-TTS outperforms existing pseudo-streaming models on most subjective and objective metrics. Further analysis shows that causal commitment stabilizes online segmentation and reduces failures caused by insufficient context, while speech-state inheritance improves boundary continuity without degrading naturalness or speaker identity. X2Streaming-TTS thus achieves strict token-level synthesis with quality comparable to the evaluated offline baselines, a median time to first audio token (TTFT) of 15.8\,ms for a single request, and a median TTFT of 260.8\,ms at 128 concurrent requests. Our implementation is publicly available at \url{https://github.com/X-Square-Robot/X2Streaming-TTS}.
\end{abstract}

\section{Introduction}

Once speech is played, it cannot be revised. When text is generated token by token by an upstream language model, this irreversibility becomes a fundamental constraint rather than merely an engineering concern: the TTS system must decide when to speak based solely on the observed text prefix, while future tokens cannot alter audio that has already been delivered. For example, suppose the language model has emitted the prefix \texttt{He finished 3}. It may subsequently become \texttt{He finished 3 laps}, in which \texttt{3} is pronounced as \emph{three}, or \texttt{He finished 3rd}, in which it is pronounced as \emph{third}. Although the two pronunciations begin similarly, they diverge before either word is complete. A system that has already emitted \emph{three} therefore cannot revise it into \emph{third}. Offline synthesis does not encounter this decision, nor does a system that buffers an entire sentence before synthesizing it, because the relevant text has already been fixed in both cases.

Three requirements that are straightforward in offline TTS become difficult to satisfy simultaneously in token-level streaming. First, pronunciation: incomplete numbers, units, and symbols may be reinterpreted by characters that arrive after they have already been voiced. Second, breathing room: speakers pause where permitted by the language, but punctuation-only segmentation produces numerous short segments that expend the generation budget on repeated stopping and restarting, whereas a fixed window fills the budget by cutting at positions unrelated to linguistic structure. Third, continuation: a segment generated from silence must re-establish pitch and timbre from scratch, making the resulting seam perceptible.

Taken together, these three challenges show that token-level streaming is not merely a matter of reducing first-audio latency. More fundamentally, it requires the system to make irreversible commitments under partial observability: each decision produces output that cannot be revised, even though its correctness may depend on information that has not yet arrived. Simultaneous translation emits target words before the source sentence ends \citep{ma2019stacl,ma2019mma}, streaming recognition displays partial hypotheses \citep{graves2012rnnt,shi2020emformer}, and long-video generation cannot re-render frames it has produced \citep{henschel2024streamingt2v,yin2025causvid}. All share one structure: \emph{when may I commit}, and \emph{how do I continue the state I have produced}.

We address these questions for streaming speech synthesis with two complementary mechanisms. \emph{Causal commitment} governs decisions up to and including segment closure, whereas \emph{causal speech-state inheritance} governs state propagation across segment boundaries. Removing either mechanism leaves an observable deficiency: without speech-state inheritance, a measurable pitch discontinuity appears at segment boundaries; without causal commitment, generation may exhaust its capacity and force segmentation at linguistically inappropriate positions. Figure~\ref{fig:arch} maps these challenges onto the synthesis pipeline: the left panel traces token-level text release and audio generation, while the right panel shows first-audio latency, irreversible commitment under prefix ambiguity, and cross-segment discontinuity together with the mechanisms that address them. A streaming frontend determines which observed text can be released, the Talker converts the released text into discrete acoustic tokens, and Code2Wav decodes these tokens into a waveform. Causal commitment determines when the current segment closes, while causal speech-state inheritance determines the state from which the next segment continues. Our contributions are as follows:

\begin{itemize}
\item We introduce a causal-commitment mechanism that keeps ambiguous numbers, units, and symbols provisional until their pronunciations are resolved, while formulating capacity and boundary admission as a unified constrained segmentation problem.
\item We introduce a causal speech-state inheritance mechanism that transfers the complete waveform-decoder state and a fixed number of acoustic states through two independent paths, using a fixed causal prior that assigns zero weight to future positions and an explicitly bounded residual.
\item We develop an end-to-end causal TTS system that integrates these two mechanisms to perform strict token-level synthesis without access to future text. Experimental results demonstrate synthesis quality comparable to that of the evaluated offline baselines. Median time to first audio token (TTFT) is 15.8\,ms for a single request and 260.8\,ms at 128 concurrent requests.
\end{itemize}

\begin{figure}[!t]
\centering
\includegraphics[width=\textwidth]{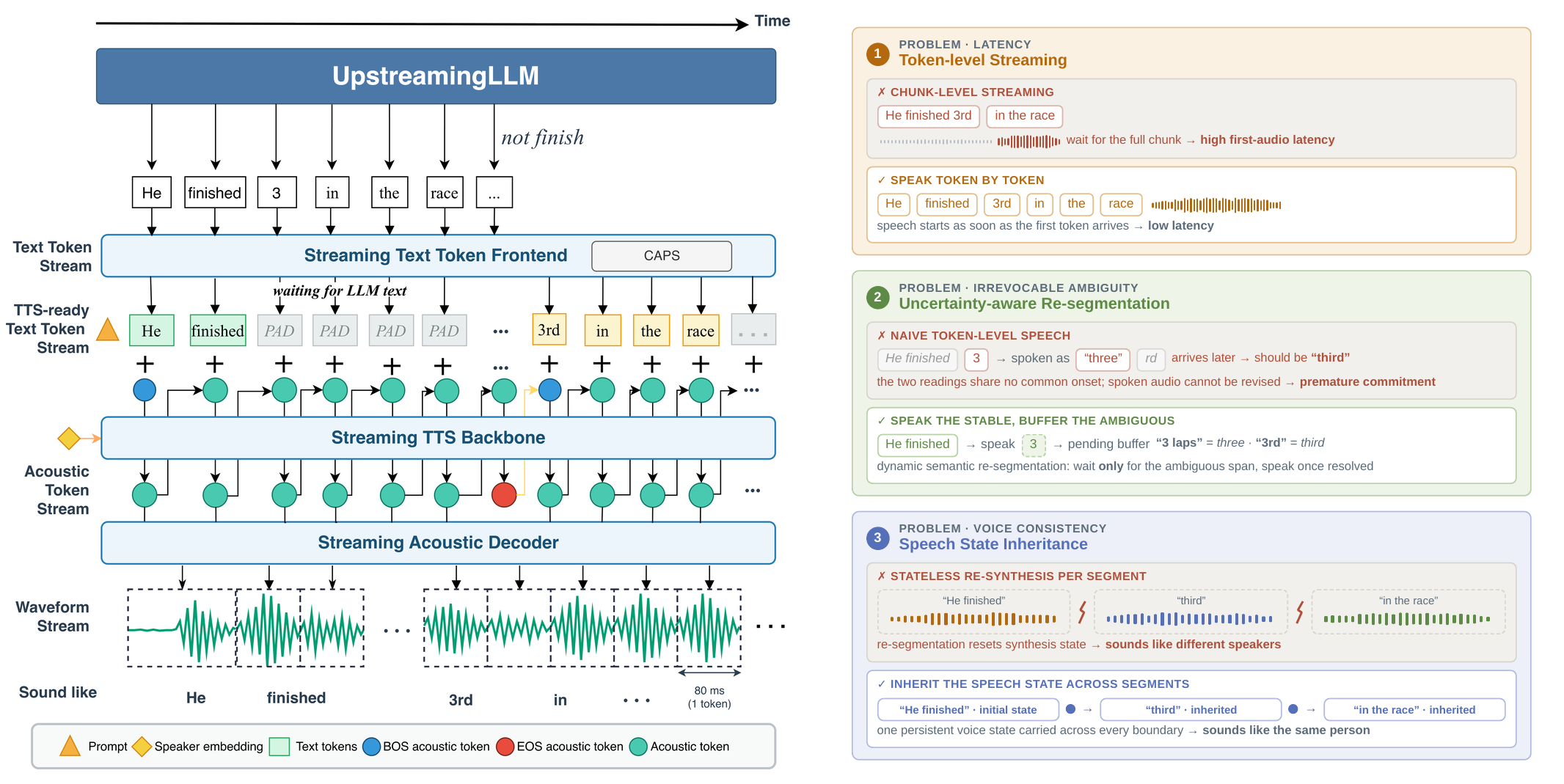}
\caption{\textbf{Left:} Streaming pipeline under token-level arrival. The frontend releases only TTS-ready text (\texttt{PAD} when nothing is ready; ambiguous spans such as \texttt{3} in \texttt{He finished 3} are held until pronunciation is resolved), the Talker emits acoustic tokens, and Code2Wav streams waveform for immediate playback. \textbf{Right:} Three challenges and the corresponding mechanisms---(1) first-audio latency via token-level rather than chunk-level start; (2) irreversible early commitment via uncertainty-aware buffering; (3) cross-segment discontinuity via speech-state inheritance.}
\label{fig:arch}
\end{figure}

\section{Related Work}

We organize prior work by the text context available at synthesis onset and by how systems
handle boundary selection, acoustic capacity, and cross-segment state.

\paragraph{Fully observed text, streaming audio out.} These methods begin synthesis after a complete sentence becomes available and reduce output latency through autoregressive acoustic generation, causal or chunk-aware decoding, and streaming vocoders \citep{du2024cosyvoice,du2024cosyvoice2, du2025cosyvoice3,xie2025fireredtts2,kong2020hifigan,siuzdak2023vocos}. They build on discrete audio codecs \citep{zeghidour2021soundstream,defossez2022encodec} and language-model-based zero-shot TTS architectures \citep{wang2023valle,lajszczak2024basetts,ye2025llasa, chen2024f5tts}. Although they enable concurrent audio generation and playback, waiting for a complete sentence introduces input-side latency and does not address pronunciation or segmentation under uncertain text prefixes.

\paragraph{Streaming text with limited lookahead.} LiveSpeech, SyncSpeech, and SpeakStream
reduce the required context through fixed lookahead, token-synchronous decoding, and
text--speech interleaving
\citep{dang2024livespeech,sheng2025syncspeech,bai2025speakstream}. Related work has also
investigated prosodic-boundary and pause prediction
\citep{dai2022prosodyannotation,yang2023pauseinsertion,yang2025phrasebreak}. Boundary-aware
streaming generation \citep{liu2026boundaryaware} uses limited future words, boundary
post-training, and a bounded sliding prompt, while MagpieTTS-LF
\citep{ghosh2026magpiettslf} combines punctuation-aware chunks, historical text,
attention-tracking states, and an inference-time soft prior. These methods reduce sentence-level
waiting but still depend on future tokens, so their latency remains coupled to the upstream
generation rate and they do not satisfy strict zero-lookahead causality.

\paragraph{Segmentation without acoustic cost.} Standalone text segmenters
\citep{frohmann2024sat} can identify linguistically appropriate boundaries but do not account
for acoustic generation costs. A suitable boundary may therefore appear only after the
acoustic budget has been exhausted. We use SaT-3L, which accesses up to 48 future subwords,
only as a boundary-quality reference. Moreover, text segmentation does not define what speech
state should be transferred across a boundary and therefore cannot ensure cross-segment
continuity.

Existing methods thus wait for complete text, rely on future context, or ignore acoustic
capacity and cross-segment state; our approach instead couples linguistic readiness with
acoustic cost under strict causality and transfers bounded speech state across committed
boundaries.

\section{X2Streaming-TTS}

Zero lookahead is affordable because text is information-dense, whereas
speech is temporally redundant. We quantify this difference using the
acoustic-to-text expansion ratio $\rho$, defined as the number of
autoregressive acoustic decoding steps per text token. For a completed
segment $S_k$ containing $n_k$ text tokens and requiring $A_k$ acoustic
decoding steps, its realized expansion ratio is
$\rho_k=A_k/n_k$. Because $\rho_k$ is typically well above one, voicing
the currently available text provides time for subsequent tokens to
arrive. Lookahead-based methods consume future text and therefore require
the upstream model to wait. In contrast, our method exploits the time
already available on the acoustic side. After each completed segment,
the system updates an online estimate $\widehat\rho_k$, turning the
question ``when may I speak'' into a causal admission test.
Let $\tau$ denote wall-clock time, and let $N(\tau)$ be the number of text
tokens observed by time $\tau$. Let $y_u$ denote the $u$-th discrete
acoustic token, generated at time $\tau_u$, and let $h_{k-1}$ denote the
speech state inherited from the preceding segment. Strict causality
requires
\begin{equation}
y_u \perp x_{N(\tau_u)+1:}
\mid x_{1:N(\tau_u)},y_{<u},h_{k-1},
\label{eq:strict-causality}
\end{equation}
that is, an acoustic token may depend only on the text observed before
its generation time, previously generated acoustic tokens, and inherited
speech state.
Implementing this constraint requires a backbone with three properties:
countable acoustic output, transferable decoding state, and observable
remaining capacity. Qwen3-TTS \citep{hu2026qwen3tts} provides all three.
Its Talker emits countable discrete acoustic tokens, its Code2Wav module
maintains a waveform-decoder state decoupled from the linguistic
component, and its KV-cache-aware decoder exposes the remaining capacity
in cache positions.

\subsection{Causal Commitment}

Causal commitment determines which portion of the arrived text can be irreversibly released
for synthesis and when the current segment should be closed. It consists of two coupled
components. Uncertainty-aware semantic readiness prevents expressions whose pronunciations
may still change from being released prematurely, while capacity-adaptive punctuation-aware
segmentation closes the released prefix before the acoustic generation budget is exhausted.

\subsubsection{Uncertainty-Aware Semantic Readiness}

Let $X_t=x_{1:t}$ denote the text observed after the arrival of token
$x_t$. The frontend maintains a pending suffix $U_t$ whose pronunciation
may still be altered by future tokens. Let $\oplus$ denote token-sequence
concatenation. Upon receiving $x_t$, the frontend partitions the buffered
text as
\begin{equation}
U_{t-1}\oplus x_t=E_t\oplus U_t,
\label{eq:readiness-partition}
\end{equation}
where $E_t$ is the longest newly eligible prefix and $U_t$ is the
remaining unresolved suffix. Deterministic closure rules cover numbers,
units, symbols, abbreviations, and punctuation. A span is released only
when no continuation admitted by these rules can change its
pronunciation; once released, it is never revised.
For example, after observing \texttt{He finished 3}, the frontend retains \texttt{3} in
$U_t$. When subsequent tokens resolve the expression, the complete span is normalized and
released atomically, ensuring that the acoustic model never receives a partial normalization.
At the end of input, any remaining suffix is treated as closed and normalized using the same
rules. Because both $E_t$ and $U_t$ depend only on $X_t$, this procedure is strictly causal.
Semantic readiness and capacity control operate sequentially. Only ready, normalized units are
admitted to the active TTS segment, while $U_t$ remains outside it. If the active segment
approaches its estimated capacity before $U_t$ is resolved, the system closes the preceding
stable prefix without splitting the unresolved expression. Once resolved, the expression is
admitted atomically to the current segment if sufficient capacity remains, or otherwise to the
next segment. Thus, semantic readiness determines \emph{what} may be committed, whereas
capacity-adaptive segmentation determines \emph{when} the active segment must close.

\subsubsection{Capacity-Adaptive Punctuation-Aware Segmentation}
\paragraph{Problem and offline reference.}
Let $x_{1:n}$ be the text tokens. A boundary after $x_i$ has type
$q_i\in\{1,2,3,Hard\}$, where tiers $1$--$3$ denote sentence-final,
clause-level, and weaker punctuation, respectively, and $Hard$ means a hard
non-punctuation break, with
$0\le c(1)<c(2)<c(3)<c(Hard)$. A partition $\pi=(0=b_0<\cdots<b_m=n)$ induces segments
$S_k=x_{b_{k-1}+1:b_k}$. Given synthesis context $\xi$, let
$G(S;\xi)$ denote the realized KV-cache footprint of segment $S$,
measured in cache positions, and let $B$ denote the usable cache
capacity. Let $\lambda_{\rm seg}\ge0$ be the penalty associated with
creating an additional segment. The hindsight problem is
\begin{equation}
\begin{aligned}
\pi^\star_{\rm off}\in\arg\min_\pi
\left\{\lambda_{\rm seg}m+\sum_{k=1}^{m-1}c(q_{b_k})\right\}\\[4pt]
\text{s.t.}\quad G(S_k;\xi)\le B\quad\forall k.
\label{eq:offline-objective}
\end{aligned}
\end{equation}
The natural end has no boundary penalty, and streaming latency is
evaluated separately. If span cost and feasibility depend only on its
endpoints, the exact offline solution follows from
\begin{equation}
\begin{aligned}
V(0)&=0,\\[4pt]
V(j)&=\min_{\substack{0\le i<j\\G(x_{i+1:j};\xi)\le B}}
\left[V(i)+\lambda_{\rm seg}+\mathbf 1\{j<n\}c(q_j)\right].
\label{eq:offline-dp}  
\end{aligned}
\end{equation}
Backtracking returns the optimal boundaries. This takes $O(n^2)$
time, or $O(nC)$ if a segment contains at most $C$ tokens. If inherited
state affects feasibility, it must be included in the DP state. In
evaluation, we use either measured $G$ or one frozen resource predictor
in the feasibility test.


\paragraph{Delayed-feedback capacity adaptation.}
For a completed segment $S_k$ of length $n_k$, let $P_k$ be the cache
length at autoregressive decoding onset and $A_k$ the subsequent
decoding length, including EOS and draining; hence
$G(S_k;\xi)=P_k+A_k$. CAPS updates the projected EMA
\begin{equation}
\widehat \rho_{k+1}
=\Pi_{[\rho_{\min},\rho_{\max}]}
\left((1-\beta_k)\widehat \rho_k+\beta_k\frac{A_k}{n_k}\right),
\label{eq:ema}
\end{equation}
where a larger $\beta_k$ can be used after an overflow.
With prefill estimate $\widehat P_k$ and headroom $R$, it uses
\begin{equation}
\widehat G_k(S;\xi)=\widehat P_k+\widehat \rho_k|S|,\qquad
\widehat C_k=\left\lfloor
\frac{B-R-\widehat P_k}{\widehat \rho_k}\right\rfloor .
\label{eq:capacity}
\end{equation}
Here, $\widehat C_k$ is the predicted text-token capacity of segment
$k$, namely the maximum number of tokens CAPS allows before a forced
split under the current resource estimate.
We assume $B-R-\widehat P_k\ge\widehat \rho_k\ge1$. When a segment opens,
its estimate, capacity, and thresholds are frozen; feedback affects
only segments opened afterward.

\paragraph{Causal punctuation-aware rule.}
Choose $0<\alpha_1\le\alpha_2\le\alpha_3\le1$ and define
\[
T_{\ell,k}
=
\left\lceil\alpha_\ell\widehat C_k\right\rceil,
\qquad \ell\in\{1,2,3\}.
\]
After appending $x_t$, let $L_t$ denote the number of text tokens in
the active segment. CAPS closes the segment after $x_t$ if and only if
\begin{equation}
\left[
\bigvee_{\ell=1}^{3}
\left\{q_t=\ell\wedge L_t\ge T_{\ell,k}\right\}
\right]
\quad\text{or}\quad
L_t\ge\widehat C_k.
\label{eq:stopping-rule}
\end{equation}
The second condition implements a tier-4 hard boundary when no eligible
punctuation boundary is encountered before capacity is reached.

\paragraph{Fragmentation bound and scope.}
To analyze the number of segments, suppose that each segment can
contain at most $C$ tokens. A sequence of $n$ tokens then requires at
least
\[
m_C^\star=\left\lceil\frac{n}{C}\right\rceil
\]
segments, and this minimum is achieved by cutting after every $C$
tokens. Under the CAPS stopping rule, every non-final segment contains
at least $\lceil\alpha_1 C\rceil$ tokens.
\begin{theorem}[Fragmentation bound]
If $\widehat C_k=C$ for all $k$ and every non-final segment is closed
by Eq.~\eqref{eq:stopping-rule}, then
\begin{equation}
m_{\rm CAPS}\le
\frac{C}{\lceil\alpha_1 C\rceil}m_C^\star+1
\le\frac{1}{\alpha_1}m_C^\star+1.
\label{eq:fragmentation-bound}
\end{equation}
\end{theorem}
\begin{proof}
Let $L=\lceil\alpha_1C\rceil$. CAPS has $m_{\rm CAPS}-1$ non-final
segments, each containing at least $L$ tokens, so
$(m_{\rm CAPS}-1)L\le n$. The benchmark has $m_C^\star$ segments, each
holding at most $C$ tokens, so $n\le Cm_C^\star$. Therefore
\[
m_{\mathrm{CAPS}}
\le \frac{n}{L}+1
\le \frac{C}{L}m_C^\star+1
\le \frac{1}{\alpha_1}m_C^\star+1.
\qedhere
\]
\end{proof}
The theorem concerns segment count only. The additive $1$ accounts
for a possibly short final segment. It does not guarantee punctuation
quality or physical KV-cache safety. A predicted-feasible segment is
guaranteed to fit the cache if
\[
G(S;\xi)\le\widehat G_k(S;\xi)+R_{\rm cap}.
\]
Since an EMA does not ensure this condition for every segment, we
evaluate prediction error, overflow rate, partition cost, prosody,
and latency empirically.

\paragraph{Implementation details.}
We implement CAPS as a token-driven finite-state machine in the
frontend. Sentence-final marks form tier 1, commas, semicolons, and
colons tier 2, and line breaks, ellipses, and dashes tier 3; trailing
quotes and brackets are ignored when assigning the tier. The cache
limit $B$ is read from the loaded engine. When a segment opens, CAPS
computes and freezes its capacity and thresholds, then tests
Eq.~\eqref{eq:stopping-rule} for every incoming token. At segment
completion, the backend reports the decoding steps and text-token
count used in Eq.~\eqref{eq:ema}; the update applies only to segments
opened later. Unless stated otherwise, we initialize $\widehat\rho_1=6.0$, use
$\beta_k=0.1$ normally and $\beta_k=0.5$ after overflow, and clip
$\widehat\rho_k$ to $[2,10]$. We set 
$(\alpha_1,\alpha_2,\alpha_3)=(0.7,0.8,0.9)$; hence the segment-count
factor in Eq.~\eqref{eq:fragmentation-bound} is at most
$10/7\approx1.429$. We allow at most two text segments to be queued,
while acoustic decoding remains sequential so that segment $S_k$
starts only after the inherited state of $S_{k-1}$ becomes available.

\subsection{Causal Speech-State Inheritance}

\begin{figure}[!t]
\centering
\includegraphics[width=0.52\textwidth]{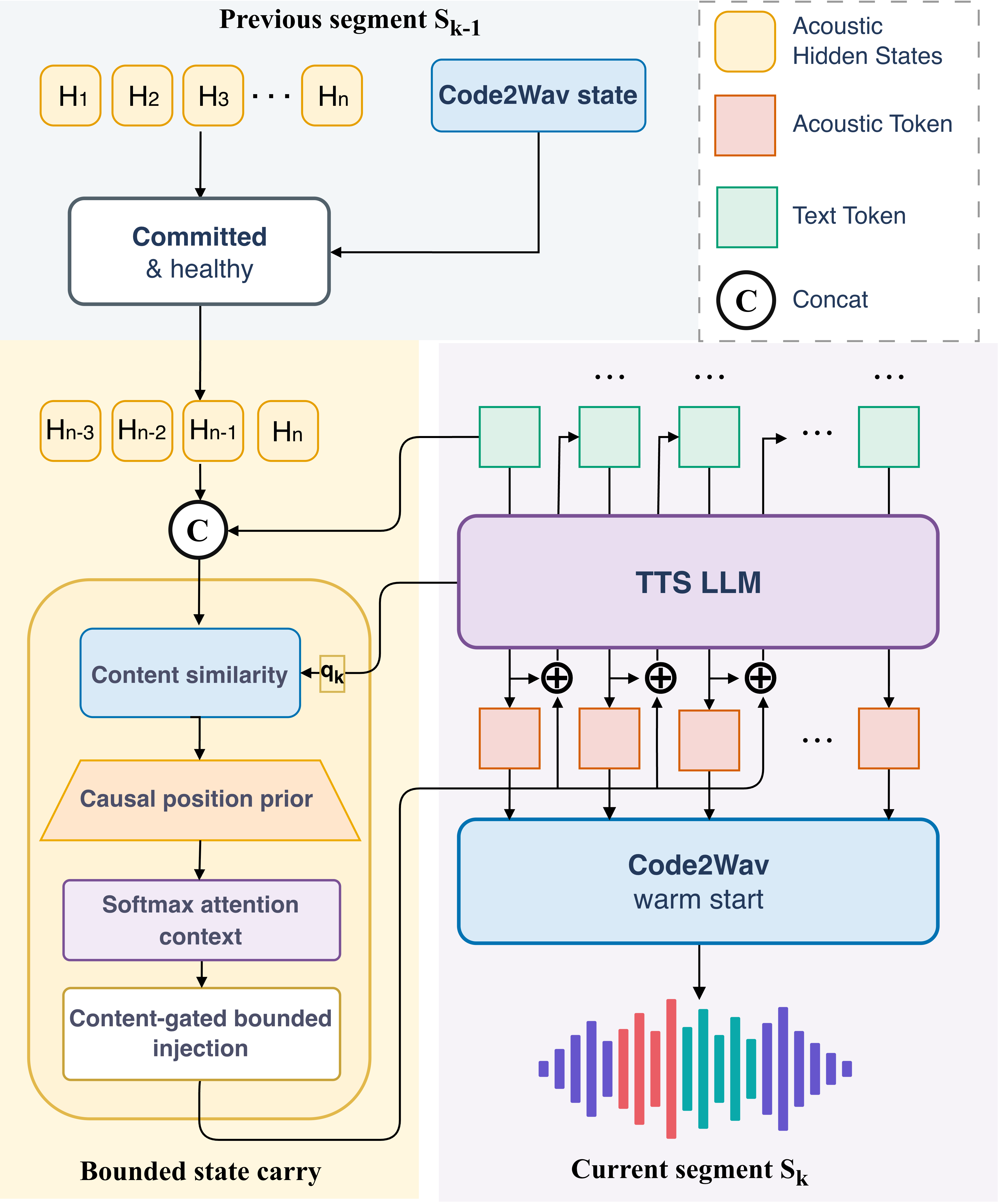}
\caption{Speech-state inheritance across a segment boundary. After a
health check on $S_{k-1}$, its Code2Wav state warm-starts $S_k$, while
its trailing Talker states provide bounded historical context through
causal-prior attention and a gated residual.}
\label{fig:inherit}
\end{figure}

The end of a segment does not mark the end of its speech state. If the next segment starts from
a zero state, the waveform decoder must rebuild its cache, while the acoustic component must
re-establish local pitch, energy, and timbre. These resets can produce a measurable
discontinuity at the segment boundary. Figure~\ref{fig:inherit} illustrates how the speech
state is transferred across such a boundary.

\subsubsection{Two independent state paths.}
After segment $S_{k-1}$ ends, the system stores and transfers the
complete Code2Wav state bundle, including its KV cache, convolution
states, transposed-convolution states, and frame index. This bundle
warm-starts waveform decoding for segment $S_k$.
In parallel, the system transfers the trailing $H$ token-aligned Talker
states as bounded historical memory. Here, $H$ is a fixed global
hyperparameter independent of the stream length; its numerical value is
reported in the experimental setup. This fixed history bounds the
inherited acoustic context by $O(H)$.

A health check determines whether the inheritance chain remains active.
Both state paths are retained only if the predecessor received complete
input, consumed all admitted text, terminated acoustic decoding normally,
produced a realized expansion ratio
$\rho_k=A_k/n_k\in[1,12]$, and yielded an intact Code2Wav snapshot.
Otherwise, both inherited states are cleared and generation restarts
from the default state.

\subsubsection{Fixed causal attention prior.}
As shown in Figure~\ref{fig:inherit}, we form the memory
$\mathcal M=[\mathbf h_{-H+1:0};\mathbf e_{1:L}]$ by concatenating the
trailing $H$ Talker states from $S_{k-1}$ with the token-aligned
representations of $S_k$ after projection into a shared space. Let
$\mathbf m_j\in\mathcal M$ denote memory position $j$ and
$d_{u,j}=u-j$ its signed distance from query $\mathbf q_u$. The attention
logits and weights are
\begin{equation}
\ell_{u,j}=2\cos(\mathbf q_u,\mathbf m_j)+b(d_{u,j}),
\qquad
a_{u,j}=\operatorname*{softmax}_{j}(\ell_{u,j}),
\label{eq:prior-attention}
\end{equation}
where
\begin{equation}
b(d)=
\begin{cases}
-d^2/8, & 0\le d\le4,\\
\log 0.1, & d\ge5,\\
-\infty, & d<0.
\end{cases}
\label{eq:causal-prior}
\end{equation}
Thus, future positions receive zero weight. For equal content
similarity, the current position has prior odds of $10\!:\!1$ over
positions at distance at least five, while relevant historical states
remain accessible through the content term.

\subsubsection{Bounded injection.}
We use normalized attention concentration as the gating score. Let
$M_u=|\mathcal A_u|$ and define
\begin{equation}
s_u=
\begin{cases}
1+\dfrac{\sum_{j\in\mathcal A_u}a_{u,j}\log a_{u,j}}
{\log M_u}, & M_u>1,\\[5pt]
1, & M_u=1.
\end{cases}
\label{eq:attention-gate}
\end{equation}
Thus, $s_u\in[0,1]$, with larger values indicating more concentrated
attention. Using the bounded gain
$g_u=0.015(0.5+0.5s_u)$, we inject the attended causal context as
\begin{equation}
\overline{\mathbf m}_u
=\sum_{j\in\mathcal A_u}a_{u,j}\mathbf m_j,
\qquad
\mathbf z_u'
=\mathbf z_u+g_u\overline{\mathbf m}_u.
\label{eq:bounded-injection}
\end{equation}
Because $0.0075\le g_u\le0.015$, the resulting perturbation satisfies
$\|\mathbf z_u'-\mathbf z_u\|
\le0.015\max_{j\in\mathcal A_u}\|\mathbf m_j\|$, explicitly bounding
the influence of the carried context.
\section{Experiments}

\input{tables/table1_main_results}

\subsection{Setup and Evaluation}

\paragraph{Data.} The segmentation and continuity evaluations use 59 held-out passages from the Mandarin proficiency subset of a Chinese TTS corpus, comprising 954 source-sentence boundaries within passages. Overall synthesis quality is evaluated using SEED-TTS-Eval \citep{anastassiou2024seedtts}, a target-speaker evaluation set, and long-text extrapolation from $1\times$ to $10\times$. Stability experiments use passages lasting hundreds of seconds and evaluate each system over 60 fixed 10\,s windows. Symbol-oriented evaluation covers numeric, streaming-ambiguity, and non-natural categories, with identical text provided to all systems and subjective ratings collected from 120 listeners. Text tokens are supplied at a fixed rate for controlled evaluation; a deployed upstream language model may exhibit variable token-generation latency. All X2Streaming-TTS runs fix the inherited Talker history at $H=4$. Serving latency is measured separately on one RTX 5090 using the deployed BF16 engine. We test 1 to 128 concurrent requests, with 20 measured rounds per level after 3 warmup rounds.

\paragraph{Metrics.} Recognition performance is measured using Whisper \citep{radford2022whisper} and Paraformer \citep{gao2022paraformer}. Speaker similarity is measured using ECAPA-TDNN \citep{desplanques2020ecapatdnn,wang2022wespeaker}, and predicted naturalness is measured using UTMOS \citep{saeki2022utmos}. To quantify discontinuity at a segment boundary, we extract 1000\,ms of audio on each side of the boundary and compute the absolute differences in mean pitch and mean energy, denoted by $\Delta F0$ and $\Delta E$, respectively. PBD is their mean after min--max normalization. Each estimate is first aggregated at the passage level and then bootstrapped over passages 10{,}000 times.

\subsection{Main Results}

The first and final rows of Table~\ref{tab:main} provide the most closely controlled
comparison. They use the same offline backbone and identical model weights, differing only in
their input and decision conditions. Their performance difference therefore reflects the
effect of strictly incremental token-level input. Our method achieves lower recognition error
than the offline reference in 3 of the 8 conditions, while its largest degradation among the
remaining 5 conditions is 0.62 percentage points. These results indicate that token-level
input introduces little measurable loss of intelligibility under the evaluated conditions.
Among the evaluated streaming systems, our method achieves the lowest error in 6 of the 8
conditions. The exceptions are the SEED EN condition, where CosyVoice 3-S obtains a WER of
1.68 compared with our 1.93, and the long-text $2\times$ condition, where it obtains 3.36
compared with our 3.67. Our method therefore achieves the lowest recognition error on most
evaluated streaming conditions.
The long-text results provide further evidence of stability. Our method obtains 2.55 at
$1\times$, outperforming all evaluated comparators, including both offline systems. Its error
increases to 4.36 at $10\times$ but remains below that of every streaming comparator. This
trend is consistent with capacity-based commitment: each segment receives a bounded generation
budget, and additional input length is accommodated by introducing further bounded segments
while inheriting speech state across their boundaries.

\subsection{Causal Commitment}

\input{tables/table2_commitment}

Table~\ref{tab:commit} evaluates two complementary aspects of causal commitment. The left block
measures how effectively each policy uses the available generation budget. Compared with
punctuation-only segmentation, our controller increases utilization from 11.77\% to 76.93\%.
Compared with the fixed-window policy, it reduces the hard-cap rate from 87.13\% to 0.54\%.
All evaluated policies produce zero estimated-budget violations and zero token-cap violations.
Although the fixed-window policy achieves higher utilization at 92.48\%, it also incurs an
87.13\% hard-cap rate, indicating that most segments are terminated where no linguistic
boundary is available. In contrast, the proposed controller achieves high utilization while
rarely requiring a non-linguistic hard cut, providing a more favorable balance between
resource usage and boundary quality.
The right block evaluates the linguistic quality of the selected boundaries. Our joint
text--acoustic estimator achieves an F1 score of 0.952, with a miss rate of 0.057 and a
false-split rate of 0.036, slightly outperforming SaT-3L at an F1 score of 0.940. The two
deterministic policies obtain substantially lower F1 scores of 0.017 and 0.199, respectively,
showing that neither acoustic capacity nor punctuation alone is sufficient to identify
appropriate linguistic boundaries.

\subsection{Causal Speech-State Inheritance}

\input{tables/table4_continuity}

Our method achieves the best value in all five columns of Table~\ref{tab:cont}. It obtains a
PBD of 0.1092, with a bootstrap interval disjoint from that of the nearest comparator,
FireRedTTS-2 at 0.1915. This separation indicates a consistent advantage at the evaluated
boundaries.
This result is also reflected in the component metrics: our method obtains a pitch
discontinuity of 22.61\,Hz and an energy discontinuity of 1.66\,dB, both lower than those of
every comparator.
FireRedTTS-2 illustrates the importance of considering boundary continuity and long-term
speaker consistency jointly. Although it achieves the strongest boundary metrics among the
three comparators, its ECAPA centroid similarity is 0.5205, compared with 0.951 for our
method, indicating substantially weaker speaker consistency over long passages. Together,
PBD and ECAPA similarity measure complementary forms of consistency: local prosodic
consistency at segment boundaries and long-term speaker consistency across passages,
respectively.

\subsection{Symbols and Prefix Ambiguity}

\input{tables/table5_symbol}

Symbol-intensive text evaluates whether the system can delay an uncertain commitment and
subsequently synthesize the resolved expression correctly. As shown in
Table~\ref{tab:symbol}, our method achieves a CER of 2.00\% and a fully correct reading rate of
73.3\%, compared with 6.65\% and 40.0\%, respectively, for the strongest comparator.
The difference between the Read and Sem.\ columns is particularly informative. Although
73.3\% of the readings are completely correct, 93.33\% preserve the intended meaning,
indicating that most remaining errors are pronunciation variants that remain understandable
to listeners. Numeric and streaming-ambiguity items both achieve 0\% CER and 100\% correct
readings. These results support retaining an incomplete expression until its pronunciation is
resolved, particularly when one additional token is sufficient to disambiguate the expression.

\subsection{Inference Latency}

\begin{figure}[!t]
\centering
\includegraphics[width=0.54\textwidth]{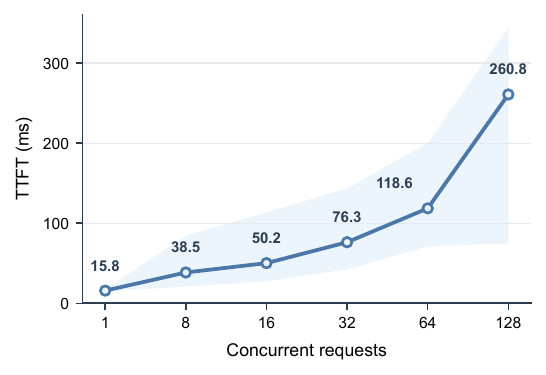}
\caption{Time to first audio token (TTFT) under concurrent requests on one RTX
5090. Points show client-side medians across the tested serving interfaces;
labels give the median values, and the shaded band spans the minimum to the
99th percentile across successful requests. Each interface was measured for 20
rounds after 3 warmup rounds, yielding $n=40$, 320, 640, 1{,}280, 2{,}560, and
5{,}120 requests at concurrency levels 1, 8, 16, 32, 64, and 128, respectively.}
\label{fig:latency}
\end{figure}
Figure~\ref{fig:latency} reports client-side TTFT as concurrency increases.
Median TTFT increases from 15.8\,ms for a single request to 38.5, 50.2, 76.3, 118.6, and
260.8\,ms at concurrency levels 8, 16, 32, 64, and 128, respectively. It remains below
120\,ms through 64 concurrent requests before rising more sharply at 128 requests. The
widening minimum-to-99th-percentile band indicates greater request-level variation under the
highest tested loads.

\section{Conclusion}

When text arrives token by token and emitted speech cannot be revised, a streaming synthesis
system must decide when the observed prefix is safe to speak and how to continue its speech
trajectory across segments. X2Streaming-TTS answers the first with causal
commitment, which holds unresolved expressions until their pronunciations are determined and
closes segments by weighing linguistic boundaries against acoustic capacity, with a provable
fragmentation bound under stated assumptions. It answers the second with causal speech-state
inheritance, transferring bounded speech state across boundaries without accessing future
positions. The resulting system matches the evaluated offline baselines in synthesis quality,
with median TTFTs of 15.8\,ms for a single request and 260.8\,ms at 128 concurrent requests.
The underlying principle is that text carries information more densely than speech: acoustic
generation buys time to process newly arrived text. The same decomposition
may extend to other irreversible online generation tasks, including simultaneous translation,
streaming recognition, and long-video generation \citep{huang2025selfforcing}.

\FloatBarrier
\bibliographystyle{aaai2027}
\bibliography{references}

\end{document}

%% file: tables/table1_main_results.tex
\begin{table*}[t]
    \centering
    \small
    \begin{tabular}{lccccccccccc}
    \toprule
    \multirow{2}{*}{Model} & \multirow{2}{*}{Streaming} & \multirow{2}{*}{Granularity}
    & \multicolumn{2}{c}{SEED-TTS-Eval} & \multicolumn{2}{c}{MiniMax}
    & \multicolumn{4}{c}{Long-text extrapolation} \\
    \cmidrule(lr){4-5} \cmidrule(lr){6-7} \cmidrule(lr){8-11}
    & & & CER$_{\mathrm{ZH}}$ & WER$_{\mathrm{EN}}$ & CER$_{\mathrm{ZH}}$ & WER$_{\mathrm{EN}}$
    & $1\times$ & $2\times$ & $5\times$ & $10\times$ \\
    \midrule
    Qwen3-TTS-12Hz-1.7B & \ding{55} & Offline
    & 1.10 & \textbf{1.43} & \underline{0.87} & \textbf{1.85}
    & \underline{3.95} & \textbf{3.05} & \textbf{4.01} & \textbf{3.93} \\
    F5-TTS & \ding{55} & Offline
    & 1.52 & 2.00 & 3.74 & 2.08
    & 4.42 & 3.81 & 5.04 & 4.95 \\
    \addlinespace[2pt]
    \midrule
    FireRedTTS-2 & \ding{51} & Chunk
    & 1.14 & 1.95 & 0.97 & 2.25
    & 5.36 & 4.79 & 5.04 & 5.32 \\
    CosyVoice 2-S & \ding{51} & Chunk
    & 1.45 & 2.57 & 1.98 & 2.38
    & 4.76 & 3.83 & 5.05 & 5.48 \\
    CosyVoice 3-S & \ding{51} & Chunk
    & \underline{0.81} & \underline{1.68} & 1.43 & 2.21
    & 4.72 & \underline{3.36} & 4.86 & 5.12 \\
    \textbf{X2Streaming-TTS(Ours)} & \ding{51} & \bf{Token}
    & \textbf{0.78} & 1.93 & \textbf{0.78} & \underline{1.86}
    & \textbf{2.55} & 3.67 & \underline{4.08} & \underline{4.36} \\
    \bottomrule
    \end{tabular}
    \caption{Intelligibility and long-text robustness (\% error; lower better). Bold best, underline second best. Granularity is the text available at the acoustic decision: Offline, Chunk, or Token.}
    
    \label{tab:main}
    \end{table*}

%% file: tables/table2_commitment.tex
\begin{table}[!t]
\centering
\small
\setlength{\tabcolsep}{3pt}
\begin{tabular}{lrrrrrr}
\toprule
& \multicolumn{3}{c}{Online replay} & \multicolumn{3}{c}{Matched budget} \\
\cmidrule(lr){2-4} \cmidrule(lr){5-7}
Policy & Seg. & Cap$\downarrow$ & Util.$\uparrow$
& F1$\uparrow$ & Miss$\downarrow$ & Split$\downarrow$ \\
\midrule
Fixed window   & 401  & 87.13 & \textbf{92.48} & 0.017 & 0.983 & 0.983 \\
Punct.\ only   & 3151 & \textbf{0} & 11.77     & 0.199 & 0.766 & 0.150 \\
Fixed ratio    & 470  & 1.70  & 78.90          & --    & --    & -- \\
SaT-3L         & --   & --    & --             & 0.940 & 0.068 & 0.050 \\
X2Streaming-TTS & 430  & 0.54  & 76.93
& \textbf{0.952} & \textbf{0.057} & \textbf{0.036} \\
\bottomrule
\end{tabular}
\caption{Causal commitment under two protocols (Cap/Util.\ in \%). Left: online replay on 59 passages (Seg.\ = segment count; Cap = hard-cap share; Util.\ = budget used). Right: matched-budget boundary quality vs.\ human annotations. Dashes: policy absent from that protocol.}
\label{tab:commit}
\end{table}

%% file: tables/table4_continuity.tex
\begin{table}[!t]
\centering
\small
\setlength{\tabcolsep}{4pt}
\resizebox{\textwidth}{!}{%
\begin{tabular}{llrrccc}
\toprule
System & Gran. & $\Delta$F0 (Hz)$\downarrow$ & $\Delta$E (dB)$\downarrow$
& PBD$\downarrow$ & ECAPA sim.$\uparrow$ & UTMOS$\uparrow$ \\
\midrule
CosyVoice 2-S & Chunk & 46.89 & 3.39
& 0.3427 [0.3226, 0.3636] & 0.9304 [0.9223, 0.9366] & 3.1949 [3.1277, 3.2582] \\
CosyVoice 3-S & Chunk & 47.53 & 3.41
& 0.3479 [0.3258, 0.3712] & 0.9264 [0.9215, 0.9319] & 2.6899 [2.6352, 2.7465] \\
FireRedTTS-2 & Chunk & 31.68 & 2.17
& 0.1915 [0.1794, 0.2041] & 0.5205 [0.4165, 0.6178]
& 3.2704 [3.0815, 3.4657] \\
\addlinespace[2pt]
X2Streaming-TTS & Token & \textbf{22.61} & \textbf{1.66}
& \textbf{0.1092} [0.1000, 0.1192] & \textbf{0.9511} [0.9454, 0.9565]
& \textbf{3.9200} [3.8388, 3.9960] \\
\bottomrule
\end{tabular}%
}
\caption{Boundary continuity and long-text stability under fixed-rate token arrival. Brackets: passage-bootstrap 95\% CIs; bold best. $\Delta$F0/$\Delta$E/PBD: $\pm1000$\,ms windows at 954 shared boundaries (59 passages). ECAPA/UTMOS: 10\,s fixed-window protocol (60 windows/system).}
\label{tab:cont}
\end{table}

%% file: tables/table5_symbol.tex
\begin{table}[!t]
\centering
\small
\setlength{\tabcolsep}{3pt}
\begin{tabular}{lrrrrr}
\toprule
System & CER$\downarrow$ & Read$\uparrow$ & UTMOS$\uparrow$ & MOS$\uparrow$
& Sem.$\uparrow$ \\
\midrule
CosyVoice 2-S & 33.41 & 0.0 & 3.075 & 3.220 & 0.00 \\
CosyVoice 3-S & 6.65  & 40.0 & 3.017 & 3.183 & 60.00 \\
FireRedTTS-2  & 18.21 & 13.3 & 2.788 & 3.629 & 6.67 \\
X2Streaming-TTS & \textbf{2.00} & \textbf{73.3} & \textbf{4.025} & \textbf{3.802}
& \textbf{93.33} \\
\bottomrule
\end{tabular}
\caption{Symbol and prefix-ambiguity evaluation on identical inputs (CER/Read/Sem.\ in \%). Read = fully correct reading; Sem.\ = listener-judged meaning preserved (120 listeners).}
\label{tab:symbol}
\end{table}